\documentclass[conference]{IEEEtran}
\IEEEoverridecommandlockouts
\usepackage{cite}
\usepackage{amsmath,amssymb,amsfonts}
\usepackage[pdftex]{graphicx}
\usepackage{epstopdf}

\usepackage{textcomp}
\usepackage{xcolor}
\usepackage{mathtools}
\def\BibTeX{{\rm B\kern-.05em{\sc i\kern-.025em b}\kern-.08em
    T\kern-.1667em\lower.7ex\hbox{E}\kern-.125emX}}

\usepackage{booktabs}       
\usepackage{nicefrac}       
\usepackage{microtype}      
\usepackage{subfig}
\usepackage{xspace}
\usepackage[linesnumbered,ruled,vlined]{algorithm2e}
\usepackage{varwidth}
\usepackage{balance}
\usepackage{multirow}
\usepackage{bm}
\usepackage{dcolumn}
\usepackage{color}
\usepackage{ulem}

\usepackage{lipsum}
\usepackage{xcolor}
\usepackage{soul}
\usepackage{url}

\usepackage{enumitem}

\usepackage{float}
\floatstyle{plaintop}
\restylefloat{table}

\newtheorem{theorem}{Theorem}

\newtheorem{proof}{Proof}
\newtheorem{definition}{Definition}

\newcommand{\method}{\textsc{P3GM}\xspace}    
\newcommand{\Add}[1]{\textcolor{black}{#1}}

\begin{document}

\title{P3GM: Private High-Dimensional Data Release via Privacy Preserving Phased Generative Model$^1$\thanks{$^1$In the version published in ICDE21, we used the Wishart mechanism for PCA~\cite{jiang2016wishart}. However, as pointed out by Cao et al.~\cite{cao2021don}, it was made clear that the Wishart mechanism does not satisfy differential privacy~\cite{wishartwrongproof}, so we change the Wishart mechanism to the Gaussian mechanism~\cite{dwork2014analyze}.
Accordingly, we changed the description about DP-PCA~(Section~\ref{subsec:dpmethods}), the privacy proof~(Section~\ref{sec:privloss}), and the experimental results~\ref{sec:exp} and Fig.~2 (e).
We confirmed that the comparison results by experiments are not affected by this change. 
Now, the results in Section~\ref{sec:exp} are the ones by the Gaussian version.}}

\author{\IEEEauthorblockN{Shun Takagi\IEEEauthorrefmark{1}\IEEEauthorrefmark{2}\IEEEauthorrefmark{4}\thanks{\IEEEauthorrefmark{1} Equal contribution. \IEEEauthorrefmark{4} A main part of the author's work was done while staying at LINE Corporation.},
Tsubasa Takahashi\IEEEauthorrefmark{1}\IEEEauthorrefmark{3}, 
Yang Cao\IEEEauthorrefmark{2} and
Masatoshi Yoshikawa\IEEEauthorrefmark{2}}
\IEEEauthorblockA{
\IEEEauthorrefmark{2}Kyoto University,
\IEEEauthorrefmark{3}LINE Corporation\\
takagi.shun.45a@st.kyoto-u.ac.jp,
tsubasa.takahashi@linecorp.com,
\{yang, yoshikawa\}@i.kyoto-u.ac.jp}}


\maketitle

\thispagestyle{plain}
\pagestyle{plain}

\begin{abstract}
How can we release a massive volume of sensitive data while mitigating privacy risks?
Privacy-preserving data synthesis enables the data holder to outsource analytical tasks to an untrusted third party.
The state-of-the-art approach for this problem is to build a generative model under differential privacy, which offers a rigorous privacy guarantee. 
However, the existing method cannot adequately handle high dimensional data.
In particular, when the input dataset contains a large number of features, the existing techniques require injecting a prohibitive amount of noise to satisfy differential privacy, which results in the outsourced data analysis meaningless.
To address the above issue, this paper proposes privacy-preserving phased generative model (P3GM), which is a differentially private generative model for releasing such sensitive data.
P3GM employs the two-phase learning process to make it robust against the noise, and to increase learning efficiency (e.g., easy to converge).
We give theoretical analyses about the learning complexity and privacy loss in P3GM.
We further experimentally evaluate our proposed method and demonstrate that P3GM significantly outperforms existing solutions.
Compared with the state-of-the-art methods, our generated samples look fewer noises and closer to the original data in terms of data diversity.
Besides, in several data mining tasks with synthesized data, our model outperforms the competitors in terms of accuracy.
\end{abstract}

\begin{IEEEkeywords}
differential privacy, variational autoencoder, generative model, privacy preserving data synthesis
\end{IEEEkeywords}

\setcounter{footnote}{1}

\section{Introduction} \label{sec:intro}

The problem of private data release, including privacy-preserving data publishing (PPDP) \cite{machanavajjhala2006diversity} \cite{sweeney2002k} and privacy-preserving data synthesis (PPDS) \cite{acs2018differentially} \cite{bindschaedler2017plausible} \cite{jordon2018pate} \cite{zhang2014privbayes}, has become increasingly important in recent years.
We often encounter situations where a data holder wishes to outsource analytical tasks to the data scientists in a third party, and even in a different division in the same office, without revealing private, sensitive information.
This outsourced data analysis raises privacy issues that the details of the private datasets, such as information about the census, health data, and financial records, are revealed to an untrusted third-party.
Due to the growth of data science and smart devices, high dimensional, complex data related to an individual, such as face images for authentications and daily location traces, have been collected.
In each example, there are many potential usages, privacy risks, and adversaries.

\begin{figure}
    \centering
    \includegraphics[width=0.9\hsize]{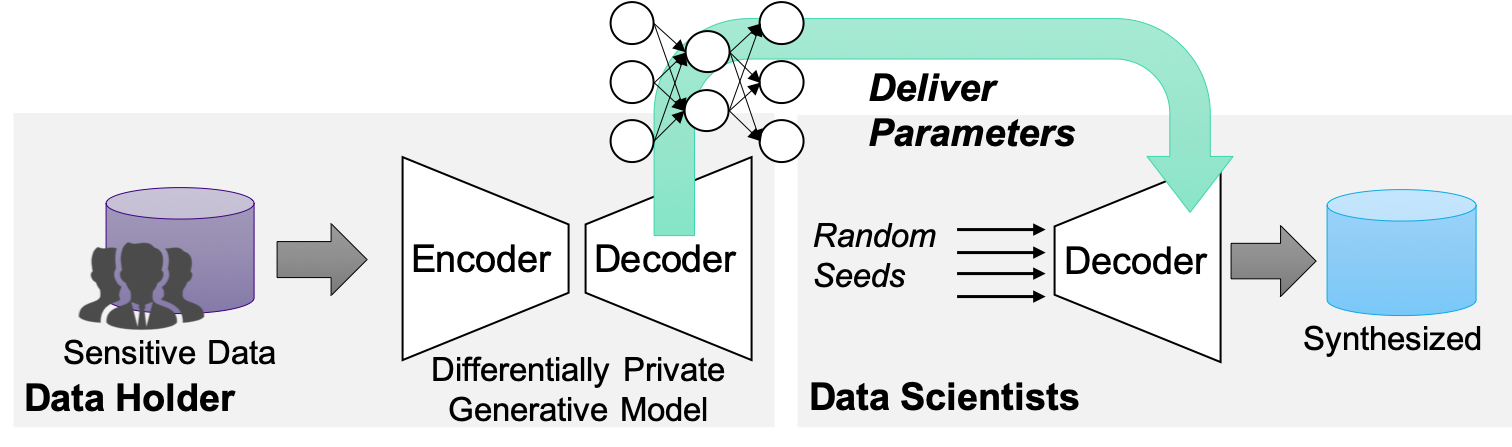}
    \caption{Privacy-preserving data synthesis via sharing a decoder of a differentially private generative model.}
    \label{fig:overview1}
\end{figure}

For the PPDP, a traditional approach is to ensure $k$-anonymity \cite{sweeney2002k}. There are lots of anonymization algorithms for various data domains \cite{abul2008never} \cite{domingo2020general} \cite{machanavajjhala2006diversity}. However, $k$-anonymity does not take into account adversaries' background knowledge.

For releasing private statistical aggregates, differential privacy (DP in short) is known as the golden standard privacy notion \cite{dwork2006differential}. 
Differential privacy seeks a rigorous privacy guarantee, without making restrictive assumptions about the adversary.
Informally, this model requires that what can be learned from the released data is approximately the same, whether or not any particular individual was included in the input database.
Differential privacy is used in broad domains and applications \cite{bindschaedler2017plausible} \cite{chaudhuri2019capacity} \cite{papernot2016semi}. 
The importance of DP can be seen from the fact that US census announced '2020 Census results will be protected using “differential privacy,” the new gold standard in data privacy protection' \cite{abowd2018us} \cite{uscensus2020}.

Differentially private data synthesis (DPDS) builds a generative model satisfying DP to produce privacy-preserving synthetic data from the sensitive data.
It has been well-studied in the literature \cite{acs2018differentially} \cite{chen2015differentially} \cite{jordon2018pate}  \cite{xie2018differentially} \cite{zhang2014privbayes} \cite{zhang2016privtree}.
DPDS protects privacy by sharing a differentially private generative model to the third party, instead of the raw datasets (Figure \ref{fig:overview1}).
In recent years, NIST held a competition in which contestants proposed a mechanism for DPDS while maintaining a dataset’s utility for analysis \cite{nistppds2018}.

\begin{figure*}
    \centering    
    \subfloat[MNIST (original)]{
        \includegraphics[width=0.19\hsize]{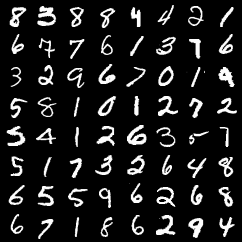}
        \label{fig:org_mnist}
    }    
    \subfloat[VAE \cite{kingma2013auto}]{
        \includegraphics[width=0.19\hsize]{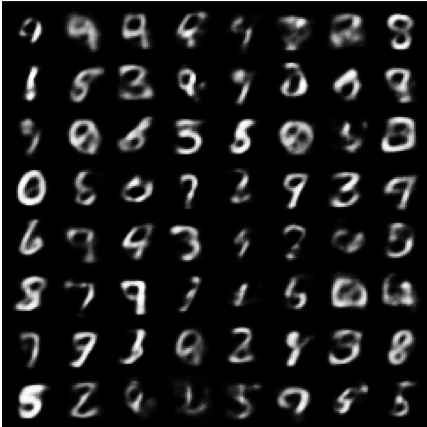}
        \label{fig:vae_mnist}
    }
    \subfloat[DP-VAE (VAE w/ DP-SGD)]{
        \includegraphics[width=0.19\hsize]{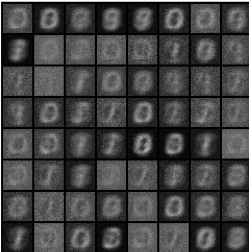}
        \label{fig:dpvae_mnist}
    }
    \subfloat[DP-GM \cite{acs2018differentially}]{
        \includegraphics[width=0.19\hsize]{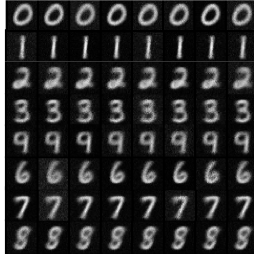}
        \label{fig:dpgm_mnist}
    }
    \subfloat[Proposed Method]{
        \includegraphics[width=0.19\hsize]{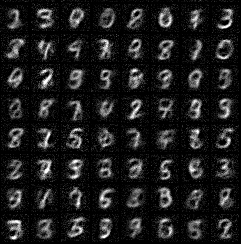}
        \label{fig:p3gm_mnist}
    }
    \caption{Sampled images from (b) VAE \cite{kingma2013auto}, (c) DP-VAE, (d) DP-GM \cite{acs2018differentially} and (e) proposed method \method. These four models are trained with (a) MNIST. (b), (c) and (d) satisfies ($1, 10^{-5}$)-differential privacy. Comparing the images sampled from DP-VAE and DP-GM, \method generates finer and more diverse samples. \method generates images that are visually closer to (a) and (b).}
    \label{fig:crown_jewel}
\end{figure*}

\begin{table*}[t]
    \centering
    \small
    \begin{tabular}{lccccc}
    \toprule
     & PrivBayes~\cite{zhang2014privbayes} 
     & \Add{Ryan's~\cite{kenna2019nistwinner}}
     & VAE with DP-SGD
     & DP-GM~\cite{acs2018differentially}
     & \textbf{\method (ours)} \\
    \cmidrule(lr){1-1}
    \cmidrule(lr){2-6}
    PPDS under differential privacy     & \checkmark & \checkmark\footnotemark[1] & \checkmark & \checkmark & \checkmark \\
    \Add{Utility in classification tasks}           & \checkmark & &        &            & \checkmark \\
    Capacity for high dimensional data  & & \checkmark &          & \checkmark & \checkmark \\
    \bottomrule
    \end{tabular}
    \caption{Contrast with competitors. Only the proposed method achieves all requirements in PPDS for high dimensional data.}
\label{tbl:salesman}
\end{table*}

To preserve utility in data mining and machine learning tasks, a generative model should have the following properties: 
1) data generated by the generative model follows actual data distribution; and 
2) it can generate high dimensional data.
However, the existing DPDS algorithms are insufficient for high dimensional data.
When the input dataset contains a large number of features, the existing techniques require injecting a prohibitive amount of noise to satisfy DP.
This issue results in the outsourced data analysis meaningless.

We now explain the existing models and their issues summarized in Table \ref{tbl:salesman}.
DPDS has been studied in the past ten years.
Traditional approaches are based on capturing probabilistic models, low rank structure, and learning statistical characteristics from original sensitive database \cite{chen2015differentially} \cite{zhang2014privbayes} \cite{zhang2016privtree}.
PrivBayes \cite{zhang2014privbayes} is a generative model that constructs a Bayesian network with DP guarantee.
However, since PrivBayes only constructs the Bayesian network among a few attributes, it is not suitable for high dimensional data. 
\Add{Ryan~\cite{kenna2019nistwinner}\footnote{\Add{Ryan's algorithm requires public information of the dataset to get relationships of correlations. We use a part of the dataset to get them without privacy protection, following the open source code~\cite{kenna2019nistwinner}.}}
was the winner of the NIST's competition by outperforming PrivBayes. 
Ryan's algorithm includes two steps to synthesize a dataset: 1) measures a chosen set of 1, 2, and 3-way marginals of the dataset with Gaussian mechanism, and 2) makes a dataset that has those marginals.
We note that this algorithm requires public information to get the crucial set of 1, 2, and 3-way marginals.}

Deep generative models have been significantly improved in the past few years.
According to the advancement, constructing deep generative models under differential privacy is also a promising direction.
We have two distinguished generative models: generative adversarial nets (GAN) \cite{goodfellow2014generative} and variational autoencoder (VAE) \cite{kingma2016improved} \cite{kingma2013auto}.

GANs can generate high quality data by optimizing a minimax objective.
However, it is well known that samples from GANs do not fully capture the diversity of the true distribution due to mode collapse.
Furthermore, GANs are challenging to evaluate, and require a lot of iterations to converge.
Therefore, under differential privacy, such learning processes tend to inject a vast amount of noise.
Actually, existing GAN based models with DP have significant limitations.
DP-GAN \cite{xie2018differentially} needs to construct a generative model for each digit on MNIST, to avoid the lack of diversity due to mode collapse.
In \cite{jordon2018pate}, PATE-GAN demonstrated its effectiveness only for low-dimensional table data having tens of attributes.

In contrast, VAEs do not suffer from the problems of mode collapse and lack of diversity seen in GANs.
However, recently proposed VAEs under DP constraints are not sufficient.
A simple extension of VAE to satisfy DP is employing DP-SGD \cite{abadi2016deep}, which injects noise on stochastic gradients.
However, it also produces noisy samples (see Figure \ref{fig:dpvae_mnist}). 
It is due to that the learning process of VAE is also complicated.
DP-GM \cite{acs2018differentially} proposed a differentially private model based on VAE.
DP-GM employs a simplified process that first partitions data by $k$-means clustering and then trains disjoint VAEs for each partition.
As a result, DP-GM can craft samples with less noise, but these samples are close to the centroids of the clusters.
This means DP-GM causes mode collapse accompanied by breaking the diversity of samples so that it can generate clear samples (see Figure \ref{fig:dpgm_mnist}).
In other words, the generated data does not follow the actual distribution of data, which causes low performance for data mining tasks.
For example, DP-GM generates clear images in Figure \ref{fig:dpgm_mnist}, but the accuracy of a classifier trained with the generated images results in $0.49$.
In this paper, we study a differentially private generative model that generates \textit{diverse} samples; the generated data follows the actual distribution.

\subsection{Our Contributions}
In this paper, we propose a new generative model that satisfies DP, named \textit{privacy preserved phased generative model} (P3GM).
Using this model, we can publish the generated data in a way which meets the following requirements:
\begin{itemize}
    \item Privacy of each data holder is protected with DP.
    \item The original data can be high dimensional.
    \item The generated data approximates the actual distribution of original data well enough to preserve utility for data mining tasks.
\end{itemize}
Because of the above properties, we can use P3GM for sharing a dataset with sensitive data to untrusted third-party such as a data scientist to analyze the data while preserving privacy.

The novelty of our paper is the new generative model with two-phased training.
P3GM is based on VAE, which has an expressive power of various distributions for high dimensional data.
However, P3GM has more tolerance to the noise for DP than VAE.
Our training model is the encode-decoder model same as VAE, but the training procedure separates the VAE's end-to-end training into two phases: training the encoder and training the decoder with the fixed encoder, which increases the robustness against the noise for DP.
Training of the decoder becomes stable because of the fixed encoder.
We define objective functions for each training to maximize the likelihood of our model.
We show that if the optimal value is given in the training of the encoder, the decoder has the possibility to generate data that follows the actual distribution.
Moreover, we theoretically describe why our two-phased training works better than end-to-end training under DP.

Furthermore, we give a realization of P3GM and a theoretical analysis of its privacy guarantee.
To show that generated data preserve utility for data mining tasks, we conduct classification tasks using data generated by the above example with real-world datasets.
Our model outperforms state-of-the-art techniques~\cite{acs2018differentially,zhang2014privbayes} concerning the performances of the classifications under the same privacy protection level.

\subsection{Preview of Results}
Figure \ref{fig:crown_jewel} shows generated samples from (b) VAE \cite{kingma2013auto}, (c) VAE \cite{kingma2013auto} with DP-SGD \cite{abadi2016deep} (we call DP-VAE), (d) DP-GM \cite{acs2018differentially}, and (e) our proposed method \method.
All methods are trained from (a) the MNIST dataset, and (c), (d) and (e) satisfy ($1, 10^{-5}$)-DP.
Behind the non-private method (b), samples from DP-VAE (c) look very noisy.
Samples from DP-GM (d) are very fine, but it generates less diverse samples for each digit.
Our proposed method (e) shows less noise and well diverse samples than (c) and (d).
\method can generate images that are visually close to original data (a) and samples from the non-private model (b).
Detailed empirical evaluations with several data mining tasks are provided in the latter part of this paper.

\subsection{Related Works}

\Add{
For releasing private statistical aggregates of curated database, differential privacy is used for privatization mechanisms.
Traditionally, releasing count data (i.e., histograms) has been studied very well \cite{cormode2019constrained} \cite{kuo2018differentially} \cite{xu2013differentially}.
To release statistical outputs described by complex queries, several works addressed differentially private indexing \cite{sahin2018differentially} and query processing \cite{kotsogiannis2019privatesql} \cite{mcsherry2009privacy}.
By utilizing these querying systems tailored to DP, we can outsource data science to third parties.
However, on these systems, data analysts are forced to understand their limitations.
Our approach enables the analysts to generate data and use them freely as well as regular data analytical tasks.
}



\section{Preliminaries} \label{sec:prelim}

In this section, we briefly describe essential backgrounds to understand our proposals. 
First, we explain variational autoencoder (VAE), which is the base of our model. 
Second, we describe differential privacy (DP), which gives a rigorous privacy guarantee.
Finally, we introduce three techniques that we use in our proposed model: differentially private mechanisms for expectation-maximization (EM) algorithm, principal component analysis (PCA), and stochastic gradient descent (SGD).

Table \ref{notations} summerizes notations used in this paper.

\begin{table}[t]
 \centering
 \small
  \begin{tabular}{cl}
  \toprule
   Symbol & \multicolumn{1}{c}{Definition} \\
   \midrule
    $x,z$ & A variable of data and a latent variable.\\
    $\bf{X}=\{\rm{x}^{(i)}\}^N_{i=1}$ & A dataset where $\rm{x}$ is a data record and\\ & $N$ is the number of data records.\\
    $p_\theta(x)$ & A marginal distribution of $x$ parametrized by $\theta$.\\
    $p_\theta(z)$ & A marginal distribution of $z$ parametrized by $\theta$.\\
    $p_\theta(x|z)$ & A posterior distribution that we refer to as \\ & \textit{decoder} parametrized by $\theta$\\
    $\theta^\ast$ & The actual parameter of the generative model \\ & which generates the dataset.\\
    $q_\phi(z|x)$ & An approximate distribution of $p_\theta(x|z)$ \\ & parametrized by $\phi$. We refer to this as \textit{encoder}.\\
    $\mu_\phi(x), \sigma_\phi(x)$ & The mean and the variance of $q_\phi(z|x)$.\\
    $\tilde{x}$ & A variable generated by a generative model.\\
    $\alpha$ & The order of re\'nyi differential privacy.\\
    $f$ & A function of dimensionality reduction.\\
    $p_\theta^f(z)$ & A distribution of $f(x)$ where $x$ follows $p_\theta(x)$\\
    $r_\lambda(z)$ & A distribution which approximates $p_{\theta^\ast}^f(z)$ \\ & parameterized by $\lambda$.\\
   \bottomrule
   \vspace{3pt}
  \end{tabular}
 \caption{Table of Symbols.}
 \label{notations}
\end{table}

\subsection{Variational Autoencoder} \label{subsec:vae}

Variational autoencoder (VAE) \cite{kingma2013auto} assumes a latent variable $z$ in the generative model of $x$.
In VAE, we maxmize the marginal log-likelihood of the given dataset $\bf{X}=\{\rm{x}^{(i)}\}^N_{i=1}$.

\textbf{Variational Evidence Lower Bound.}
Introduction of an approximation $q_{\phi}(z|x)$ of posterior $p_{\theta}(z|x)$ enable us to construct variational evidence lower bound (ELBO) on log-likelihood $\log p_\theta(x)$ as
\begin{equation}
\label{equ:elbo}
\begin{split}
    \mathcal{L}_{\text{ELBO}}(x) &= \log p_\theta(x) - D_{\text{KL}}(q_\phi(z|x) || p_\theta(x|z)) \\
    &= \mathbb{E}_{q_\phi(z|x)}[\log p_\theta(x|z)] - D_{\text{KL}}(q_\phi(z|x) || p_\theta(z)) \\
    &\leq \log p_\theta(x) . \\
\end{split}
\end{equation}
$q_{\phi}(z|x)$ and $p_{\theta}(z|x)$ are implemented using a neural network and $\mathcal{L}_{\text{ELBO}}$ can be differentiable under a certain assumption so that we can optimize $\mathcal{L}_{\text{ELBO}}$ with an optimization algorithm such as SGD.

\textbf{Reparametrization Trick.}
To implement $q_{\phi}(z|x)$ and $p_{\theta}(z|x)$ as a neural network, we need to backpropagate through random sampling.
However, such backpropagation does not flow through the random samples.
To overcome this issue, VAE introduces the reparametrization trick for a sampling of a random variable $z$ following $\mathcal{N}(\mu, \sigma)$.
The trick can be described as $z = \mu + \sigma \epsilon$ where $\epsilon \sim \mathcal{N}(0, \mathbf{I})$.


\textbf{Random sampling.}
The generative process of VAE is as follows:
1) Choose a latent vector $z$. $z \sim \mathcal{N}(0, \mathbf{I})$, 
2) Generate $\Tilde{x}$ by decoding $z$. $\Tilde{x} \sim p_\theta(x|z)$.

\subsection{Differential Privacy}

Differential privacy (DP) \cite{dwork2006differential} is a rigorous mathematical privacy definition, which quantitatively evaluates the degree of privacy protection when we publish outputs.
The definition of DP is as follows:
\begin{definition}[($\varepsilon, \delta$)-differential privacy]
A randomized mechanism $\mathcal{M}:\mathcal{D}\rightarrow\mathcal{Z}$ satisfies ($\varepsilon, \delta$)-DP if, for any two input $D, D' \in \mathcal{D}$ such that $d_H(D,D')=1$ and any subset of outputs $Z \subseteq \mathcal{Z}$, it holds that
\begin{equation}
\nonumber
  \Pr[\mathcal{M}(D)\in Z] \leq \exp(\varepsilon) \Pr[\mathcal{M}(D')\in Z] + \delta .
\end{equation}
where $d_H(D,D')$ is the hamming distance between $D, D'$.
\end{definition}

Practically, we employ a randomized mechanism $\mathcal{M}$ that ensures DP for a function $m$.
The mechanism $\mathcal{M}$ perturbs the output of $m$ to cover $m$'s sensitivity that is the maximum degree of change over any pairs of dataset $D$ and $D'$.

\begin{definition}[Sensitivity]
The sensitivity of a function $m$ for any two input $D, D' \in \mathcal{D}$ such that $d_H(D,D')=1$ is:
\begin{equation}
\nonumber
    \Delta_{m} = \sup_{D, D' \in \mathcal{D}} \|m(D)-m(D')\|.
\end{equation}
where $||\cdot||$ is a norm function defined on $m$'s output domain.
\end{definition}
Based on the sensitivity of $m$, we design the degree of noise to ensure differential privacy.
Laplace mechanism and Gaussian mechanism are well-known as standard approaches.

\subsection{Compositions of Differential Privacy} \label{sec:prelim_comp}

Let $\mathcal{M}_1, \mathcal{M}_2, \dots, \mathcal{M}_k$ be mechanisms satisfying $\varepsilon_1$-, $\varepsilon_2$-, $\dots, \varepsilon_k$-DP, respectively.
Then, a mechanism sequentially applying $\mathcal{M}_1, \mathcal{M}_2 \dots, \mathcal{M}_k$ satisfies ($\sum_{i\in[k]} \varepsilon_i$)-DP.
This fact refers to \textit{composability} \cite{dwork2006differential}.

The sequential composition is not a tight solution to compute privacy loss.
However, searching its exact solution is \#P-hard \cite{murtagh2016complexity}.
Therefore, Discovering some lower bound of accounted privacy loss is an important problem for DP.
zCDP~\cite{bun2016concentrated} and moments accountant (MA)~\cite{abadi2016deep} are one of tight composition methods which give some lower bound.

R\'enyi Differential Privacy (RDP) also gives a tighter analysis of compositions for differentially private mechanisms \cite{mironov2017renyi}.
\begin{definition}
A randomized mechanism $\mathcal{M}:\mathcal{D}\rightarrow\mathcal{Z}$ satisfies ($\alpha, \varepsilon$)-RDP if, for any two input $D, D' \in \mathcal{D}$ such that $d_H(D,D')=1$, and the order $\alpha > 1$, it holds that
\begin{equation}
    \frac{1}{\alpha-1}\log {\mathbb{E}}_{z\sim \mathcal{M}(D^\prime)} \left( \frac{\Pr(\mathcal{M}(D)=z)}{\Pr(\mathcal{M}(D^\prime)=z)}\right)^\alpha\leq\varepsilon.
\end{equation}
\end{definition}
The compositions under RDP is known to be smaller than the sequential compositions.
For RDP, the following composition theorem holds \cite{mironov2017renyi}:
\begin{theorem}[composition theorem of RDP]
\label{theo:compositionRDP}
If randomized mechanisms $\mathcal{M}_1$ and $\mathcal{M}_2$ satisfy ($\alpha, \varepsilon_1$)-RDP and ($\alpha, \varepsilon_2$)-RDP, respectively, the combination of $\mathcal{M}_1$ and $\mathcal{M}_2$ satisfies ($\alpha, \varepsilon_1+\varepsilon_2$)-RDP.
\end{theorem}
Further, between RDP and DP, the following theorem holds: 
\begin{theorem}[relation between RDP and DP~\cite{mironov2017renyi}]
\label{theo:relationRDPDP}
If a randomized mechanism $\mathcal{M}$ satisfies ($\alpha, \varepsilon$)-RDP, for any $\alpha > 1$, $0<\delta<1$, $\mathcal{M}$ satisfies ($\varepsilon+\frac{\log 1/\delta}{\alpha-1}$,$\delta$)-DP.
\end{theorem}
We can see that RDP is implicitly based on the notion of MA from the following theorem \cite{wang2018subsampled}.
\begin{theorem}[relation between RDP and MA]
\label{theo:relationRDP-MA}
In the notion of MA, the $\alpha^{th}$ moment of a mechanism $\mathcal{M}$ is defined as~\cite{abadi2016deep}: 
\begin{equation}
\nonumber
MA_{\mathcal{M}}(\alpha)\coloneqq\max_{D,D^\prime}\log\mathbb{E}_{z\sim M(D)}\exp{\left(\alpha\log\frac{\Pr(\mathcal{M}(D)=z)}{\Pr(\mathcal{M}(D^\prime)=z)}\right)}
\end{equation}
Then, the mechanism $\mathcal{M}$ satisfies $(\alpha+1, MA_\mathcal{M}(\alpha)/\alpha)$-RDP.
\end{theorem}

\subsection{Differentially Private Mechanisms}
\label{subsec:dpmethods}

Here we introduce several existing techniques used in our proposed method.
We explain DP-EM~\cite{park2017dp} and privacy preserving PCA~\cite{dwork2014analyze} and DP-SGD~\cite{abadi2016deep}.

\textbf{DP-EM: Mixture of Gaussian.}
DP-EM~\cite{park2017dp} is the expectation-maximization algorithm satisfying differential privacy. 
DP-EM is a very general privacy-preserving EM algorithm which can be used for any model with a complete-data likelihood in the exponential family. 
They introduced the Gaussian mechanism in the M step so that the inferred parameters satisfy differential privacy. 
We assume that $p(x)$ follows mixture of Gaussian $p(x;\boldsymbol{\pi}, \boldsymbol{\mu}, \boldsymbol{\Sigma})=\Sigma^K_{k=1}\pi_k\mathcal{N}(x_i;\mu_k,\Sigma_k)$, where $\Sigma^K_{k=1}\pi_k=1$ and $K$ is the number of Gaussians, and we use DP-EM algorithm to estimate parameters of it while guaranteeing differential privacy. 
Let $\{\boldsymbol{\pi}, \boldsymbol{\mu}, \boldsymbol{\Sigma}\}=\{\pi_k, \mu_k, \Sigma_k\}^K_{k=1}$ denote the parameters.
Then, the M step, where parameters are updated, is as follows.
\begin{equation}
\nonumber
    \boldsymbol{\Tilde{\pi}}=\boldsymbol{\pi}+(Y_1,...,Y_K); \ \Tilde{\Sigma}_k=\Sigma_k+Z; \ \Tilde{\mu}_k=\mu_k+(Y_1,...,Y_d)
\end{equation}
where $\boldsymbol{\pi},\Sigma_k$ and $\mu_k$ are derived with the maximum likelihood estimation in its iteration. 
$Y$ and $Z$ are the Gaussian noise for differential privacy. 
The noise is scaled with the sensitivity of their parameters. 
By adding this noise, each iteration satisfies $(\varepsilon_i, \delta_i)$-DP.
When the sensitivity is $1$\footnote{We can guarantee that the sensitivity is less than $1$ by using a technique called clipping which is described at DP-SGD in Section \ref{subsec:dpmethods}}, the upper bound of $\alpha^{th}$ moment of DP-EM of each step is as follows \cite{park2017dp}:
\begin{equation}
\label{equ:ma-dp-em}
    MA_{\text{DP-EM}}(\alpha)\leq(2K+1)(\alpha^2+\alpha)/(2\sigma_e^2)
\end{equation}
where $\sigma_e$ is the parameter which decides the scale of the noise.

\textbf{Privacy preserving PCA.}
Privacy-preserving principal component analysis (PCA)~\cite{dwork2014analyze} is the mechanism for the PCA with differential privacy.
The method follows the Gaussian mechanism.
Privacy-preserving PCA satisfies $(\varepsilon,\delta)$-differential privacy by adding noise to the covariance matrix $A$ as following:
\begin{equation}
\nonumber
    \hat{A}=A+E; 
\end{equation}
where $E$ is a symmetric noise matrix, with each (upper-triangle) entry drawn i.i.d. from the Gaussian distribution with scale $\sigma_p$.
In this paper, DP-PCA denotes this PCA.
Since the $l_2$-sensitivity of $A$ is $1$ when the $l_2$-norm of the data is upper bounded by $1$, DP-PCA satisfies $(\alpha, \alpha/(2\sigma_p^2))$-RDP from \cite{mironov2017renyi}.
We can guarantee this by the clipping technique as the case of DP-EM.

\textbf{DP-SGD.}
Differentially private stochastic gradient descent \cite{abadi2016deep}, well known as DP-SGD, is a useful optimization technique for training various models, including deep neural networks under differential privacy.
SGD iteratively updates parameters of the model $\theta$ to minimize empirical loss function $\mathcal{L}(\theta)$.
At each step, we compute the gradient $\mathbf{g}_{\rm{x}}=\nabla_\theta \mathcal{L}(\theta, \rm{x})$ for a subset of examples which is called a batch.
\Add{
However, the sensitivity of gradients is infinity, so to limit the gradient's sensitivity, DP-SGD employs the gradient clipping.
The gradient clipping $\psi_C$ limits the sensitivity of the gradient as bounded up to a given clipping size $C$.
Based on the clipped gradients, DP-SGD crafts a randomized gradient $\Tilde{\mathbf{g}}$ through computing the average over the clipped gradients and adding noise whose scale is defined by $C$ and $\sigma_s$, where $\sigma_s$ is noise scaler to satisfy $(\varepsilon,\delta)$-DP.
At last, DP-SGD takes a step based on the randomized gradient $\Tilde{\mathbf{g}}$.
DP-SGD iterates this operation until the convergence, or the privacy budget is exhausted.
}

Abadi et al.~\cite{abadi2016deep} also introduced moment accountant (MA) to compute privacy composition tightly.
The upper bound of the $\alpha^{th}$ moment of DP-SGD of each step is proved by Abadi et al.~\cite{abadi2016deep} as follows:
\begin{equation}
\label{equ:ma-dp-sgd}
\begin{split}
MA_{\text{DP-SGD}}(\alpha)\leq \frac{s^2\alpha(\alpha-1)}{(1-s)\sigma_s^2}+\sum_{t=3}^{\lambda+1}\{\frac{(2s)^t(t-1)!!}{2(1-s)^{t-1}\sigma_s^t} +\\ \frac{s^t}{(1-s)^t\sigma_s^{2t}} + \frac{(2s)^t \exp{((t^2-t)/2\sigma_s^2)(\sigma_s^t(t-1)!! + t^t)}}{2(1-s)^{t-1}\sigma_s^{2t}}\}
\end{split}
\end{equation}
where $!!$ represents the double factorial and $s$ is a sampling probability: the probability that a batch of DP-SGD includes one certain data. 
In this paper, we assume that a batch is made by uniformly sampling each data, and the batch is uniformly chosen, so $s$ is $B/N$.

\section{Problem Statement}

\Add{Suppose} dataset 
$\rm{\bf{X}}$=$\{x^{(\it{i})}\}^{\it{N}}_{\it{i}=1}$ consisting of $N$ i.i.d. samples of some continuous or discrete variable $x$.
We assume that the distribution of the data $x$ is parameterized by some parameter $\theta$; each data $\rm{x}$ is sampled from $p_{\theta^\ast}(x)$ where $\theta^\ast$ is the parameter which generates the dataset.
Since the actual parameter $\theta^\ast$ includes information of the dataset $\rm{\bf{X}}$, we publish the parameter $\theta^\ast$ instead of the dataset for privacy protection.
However, the actual parameter is hidden, so we need to train the parameter using the dataset.
Moreover, this trained parameter includes private information; an adversary may infer the individual record from the trained parameter.
Then, in this paper, we consider the way of training the parameter $\theta$ with DP.

Auto-Encoding Variational Bayes (AEVB)~\cite{kingma2013auto} algorithm is a general algorithm for training a generative model that assumes a latent variable $z$ in the generative process.
In this algorithm, a distribution $q_\phi(z|x)$ which approximates $p_\theta(z|x)$ is introduced to derive $\mathcal{L}_{\text{ELBO}}$ (Equation (\ref{equ:elbo})), and parameters $\phi$ and $\theta$ are iteratively updated by an optimization method such as stochostic gradient descent (SGD) to minimize $\mathcal{L}_{\text{ELBO}}$.
VAE is one of the models that use neural networks for $q_\phi(z|x)$ and $p_\theta(x|z)$ and becoming one of the most popular generative models due to its versatility and expressive power.
Then, the naive approach, which we call DP-VAE, for our problem is to publish $\theta$ of VAE trained by the AEVB algorithm with DP-SGD as the optimization method.
Although DP-VAE satisfies DP, we empirically found that $\theta$ trained by DP-VAE was not enough for an alternative of the original dataset, as shown in Figure \ref{fig:crown_jewel}(c).
This is because the objective function of VAE (Equation (\ref{equ:elbo})) is too vulnerable to the noise of DP-SGD to train $\theta$. 
Therefore, we introduce a new model tolerable to the noise, which we call Privacy-Preserving Phased Generative Model (P3GM).
\section{Proposed Method}

We here propose a new model, named \textit{phased generative model} (PGM), and we call its differentially private version \textit{privacy-preserved PGM} (P3GM). 
PGM has theoretically weaker expressive power than VAE but has a tolerance to the noise for DP-SGD.

Section \ref{sec:overview} gives an overview of PGM.
In Section \ref{sec:encodetrain} and \ref{sec:decodetrain}, we describe each phase of our two-phase training, respectively.
In Section \ref{sec:example}, we introduce an example of P3GM.
Section \ref{sec:synthesis} gives us how to sample synthetic data from P3GM.
In Section \ref{sec:privloss}, we give the proof of privacy guarantee of the example of P3GM by introducing the tighter method of the composition of DP.


\begin{figure}[t]
    \centering
    \subfloat[VAE]{
        \includegraphics[width=\hsize]{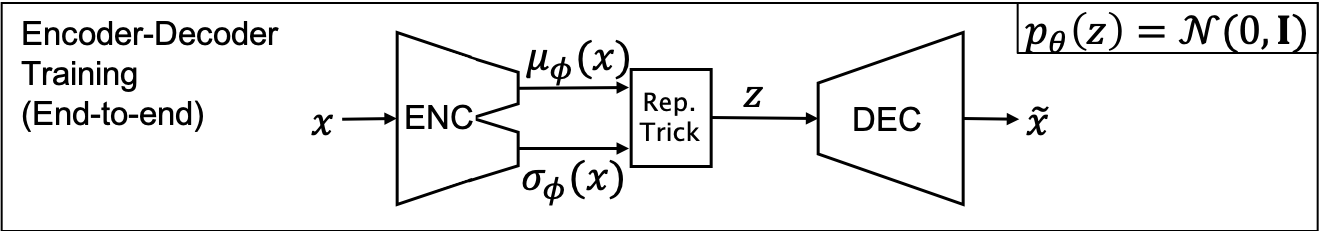}
    }\\
    \subfloat[P3GM]{
        \includegraphics[width=\hsize]{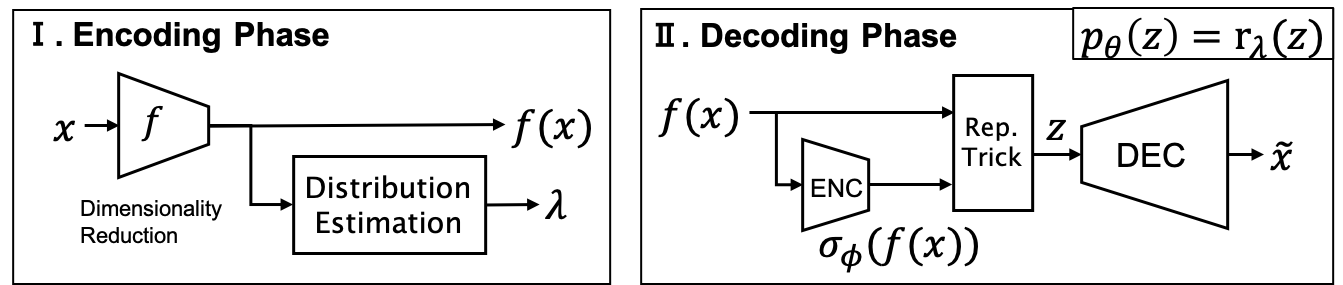}
    }
    \caption{Model architectures of VAE and P3GM.}
    \label{fig:model_overview}
\end{figure}

\newcommand{\encodetrain}{Encoding Phase}
\newcommand{\decodetrain}{Decoding Phase}

\subsection{Overview}
\label{sec:overview}

The generative model of PGM follows the same process as of VAE;
first, latent variable $z$ is generated from prior distribution $p_\theta(z)$.
Second, data $x$ is generated from posterior distribution $p_\theta(x|z)$.
Then, we introduce an approximation $q_\phi(z|x)$ of $p_\theta(z|x)$ to derive $\mathcal{L}_\text{ELBO}$ (Equation (\ref{equ:elbo})), which enables maximization of the likelihood of the given dataset. 
We will refer to $q_\phi(z|x)$ as a probabilistic \textit{encoder} because $q_\phi(z|x)$ produces the distribution over the space of $z$ given data $x$.
In a similar vein, we will refer to $p_\theta(x|z)$ as a probabilistic \textit{decoder} because $p_\theta(x|z)$ produces the distribution over the space of $x$ given a latent variable $z$.
In this paper, we assume that the encoder and the decoder produce a Gaussian distribution for the reparametrization trick and the tractability.
The main difference is in its training process shown in Figure \ref{fig:model_overview} comparing with VAE.
\Add{
Intuitively, PGM first uses a dimensional reduction $f$ instead of the embedding of VAE, which we call {\encodetrain}.
Then, PGM trains the decoder using the fixed encoder by $f$, which we call {\decodetrain}.
}
In PGM, we assume a distribution for $p_\theta(z)$ different from VAE to fix a part of parameters trained in {\decodetrain}.
Through {\encodetrain}, we can partially fix the encoder used in {\decodetrain}.
Then, we train the other parameters with the fixed encoder following the AEVB algorithm in {\decodetrain}.
The fixed encoder makes the AEVB algorithm stable even if we replace SGD with DP-SGD.
This stability is the advantage of our two-phased model.

\subsection{\encodetrain}
\label{sec:encodetrain}

Through {\encodetrain}, a part of parameters, concretely, $\mu_\phi(x)=\mathbb{E}[q_\phi(z|x)]$, becomes fixed.
In other words, {\encodetrain} partially fixes the encoder.
Here, we explain how to fix the parameter before {\decodetrain}.

The encoder's purpose is to encode original data to the latent space so that the decoder can decode the encoded data to the original data.
Another purpose of the encoder is to encode the data to a latent variable that follows some distribution so that the decoder can learn to decode the latent variable.
Therefore, we can fix the encoder by finding an encoder achieving these two purposes.

First, we describe the ideal but unrealistic assumption to make it easy to understand PGM.
The assumption is that the encoder encodes the data to the same data, which means that the encoded data is following \Add{the true distribution} $p_{\theta^\ast}(x)$.
Since the encoded data includes the same information as the original data and follows $p_{\theta^\ast}(x)$, this encoder satisfies the above two purposes.
Thus, we can fix the encoder as this.
Then, we train the decoder using the fixed encoder in {\decodetrain} while assuming that the latent variable is following $p_{\theta^\ast}(x)$.
In other words, we assume that \Add{the distribution of the latent variable $z$} $p_{\theta}(z)$ is identical to  $p_{\theta^\ast}(x)$.
We note that when the decoder is $p_{\theta}(x=\rm{x}|\it{z}=\rm{x})=1$, it holds that $p_{\theta}(x)=p_{\theta^\ast}(x)$.
This means that sampling $z$ from $p_{\theta^\ast}(x)$ and decoding $z$ to $x$ with the decoder $p_{\theta}(x|z)$, we can generate data which follows the actual distribution $p_{\theta^\ast}(x)$.

However, since \Add{the actual parameter} $\theta^\ast$ is not observed and is intractable, we cannot estimate $\theta^\ast$, encode to $p_{\theta^\ast}(x)$, and sample $z$ from $p_{\theta^\ast}(x)$.
Therefore, we approximate $p_{\theta^\ast}(x)$ by some tractable distribution to enable estimation, encoding, and sampling.
However, due to the curse of dimensionality, it is hard to infer the parameter for high dimensional data that we want to tackle.
Then, to solve this issue, we introduce a dimensionality reduction $f:\mathbb{R}^d\to\mathbb{R}^{d^\prime}$ where $d$ and $d^\prime$ are original and reduced dimensionality, respectively.
We let $p^f_{\theta}(z)$ denote the distribution of $z=f(x)$ where $x$ follows $p_{\theta}(x)$.
Then, $r_{\lambda}(z)$ denotes the approximation of $p^f_{\theta^\ast}(z)$ by some tractable distribution such as mixture of Gaussian (MoG). 
We fix the encoder to encode $x$ to data which follows $r_{\lambda}(z)$ by estimating the parameter $\lambda$.

As described above, the encoder's purpose is to encode the data so that the decoder can decode the encoded data to the original data.
From this purpose, the objective function for dimensionality reduction $f$ can be defined as follows:
\begin{equation} 
\label{equ:error_dimensional_reduction}
    \min_{f,g} \mathbb{E}_{x\sim p_{\theta^*}(x)}[\|x-g(f(x))\|^2_2]
\end{equation}
where $g$ represents a reconstruction function of $f$.
Intuitively, if there is a function $g$ where this value is small, data encoded by $f$ has the potential to be decoded to the original data.
Conversely, if this is large, the decoder will not be able to decode the encoded data to the original data.

In this model, we assume \Add{that the mean of the true distribution of the latent variable $z$ given $\rm{x}$ is $f(\rm{x})$, i.e., $\mathbb{E}[p_\theta(z|x)]=f(x)$.} 
Intuitively, this assumption means that PGM assumes that data $\rm{x}$ is generated from data whose dimensionality is reduced by $f$, i.e., $f(\rm{x})$.
This assumption enables the fixing of the encoder as $\mu_\phi(\rm{x})=\it{f}(\rm{x})$, because \Add{the encoder} is the approximation of $p_\theta(z|x)$.

\Add{Next, we estimate the distribution (i.e., $p^f_{\theta^*}(z)$) of the latent variable $z$ which is following the above generative process to feed and train the decoder with the estimated distribution.
Since the estimated distribution $r_\lambda(z)$ should approximate the distribution well, the objective function to obtain the optimal $\lambda$ can be defined as follows:
}
\begin{equation}
\label{equ:mle_lambda}
    \min_\lambda D_{\text{KL}}(p^f_{\theta^*}(z)||r_\lambda(z))
\end{equation}
where $D_{\text{KL}}$ represents the Kullback–Leibler divergence.
We consider the ideal case where there are a dimensionality reduction $f$, a reconstruction function $g$ and approximation $r_\lambda(z)$, which satisfies that Equation (\ref{equ:error_dimensional_reduction}) and Equation (\ref{equ:mle_lambda}) are $0$.
In this case, if the decoder can emulate $g$ (e.g., above example) by training, it holds that $p_\theta(x)=p_{\theta^\ast}(x)$, which means that the PGM generates data which follows the actual distribution $p_{\theta^\ast}(x)$.

We note that the variance $\sigma_\phi(x)$ of the decoder $q_{\phi}(z|x)$ is not fixed in {\decodetrain}, which means that we simultaneously train a part of the encoder for the encoder to approximate $p_{\theta^\ast}(z|x)$.

The dimensionality reduction and estimation of \Add{the distribution of the latent variable (i.e., $\lambda$)} cause privacy leak.
However, by guaranteeing DP for each component, PGM satisfies DP from the composition theorem (we refer to Section~\ref{sec:prelim_comp}).

\subsection{\decodetrain}
\label{sec:decodetrain}

We optimize the rest of the parameters of the encoder and the decoder following the AEVB algorithm. 
Here, we explain how to optimize the parameters.
$\mathcal{L}_\text{ELBO}$ on log-liklihood $\log p_\theta(x)$, which was explained in Section \ref{subsec:vae}, is approximated by a technique of Monte Carlo estimates.
\begin{equation}
\label{eq:elbo_approx}
\begin{split}
   \mathcal{L}_{\text{ELBO}}(\rm{x}) \approx &\frac{1}{L}\sum^{L}_{l=1}\log p_\theta(x=\rm{x}|\it{z}=z_{i,l}) \\
   &- D_{\text{KL}}(\textit{q}_\phi(\it{z}|\it{x}=\rm{x}) \| \textit{p}_\theta(\it{z})) 
\end{split}
\end{equation}
where $L$ is the number of iterations for Monte Carlo estimates and $z_{i,l}$ is sampled from \Add{the encoder with $\rm x$} using the reparametrization trick (we refer to Section~\ref{subsec:vae}).
If $\mathcal{L}_\text{ELBO}$ is differentiable, we can optimize parameters w.r.t. $\mathcal{L}_\text{ELBO}$ using SGD.
The first term is differentiable when we assume that $p_\theta(x|z)$ is a Bernoulli or Gaussian MLP depending on the type of data we are modeling.
Since we assume that $p_\theta(z)$ is identical to $r_\lambda(z)$, we need to choose a model for $r_\lambda(z)$ which makes the second term differentiable.

\subsection{Example of P3GM}
\label{sec:example}

We introduce a concrete realization of the privacy-preserved version of PGM, i.e., P3GM. 
Same as VAE, P3GM uses neural networks for $p_\theta(x|z)$ and $q_\phi(z|x)$.
The neural networks output the mean and variance of the distributions.

\textbf{{\encodetrain}:}
We first describe how to estimate \Add{the distribution of the latent variable} in a differentially private way.
First, we need to decide the model of the distribution.
The requirements are as follows:
\begin{enumerate}
\item The second term of Equation (\ref{eq:elbo_approx}) can be analytically calculated and is differentiable.
\item The objective function (\ref{equ:mle_lambda}) is small enough to approximately express the true distribution $p_{\theta^\ast}^f(z)$.
\item We can estimate \Add{the distribution} with DP.
\end{enumerate}
\Add{
The most simple prior distribution which satisfies the requirements is Gaussian.
However, depending on the data type, Gaussian may not be enough to approximate the true distribution (requirement 2).
Then, we introduce MoG because MoG can preserve the local structure of the data distribution more than Gaussian.
}
That is,
\begin{equation}
\nonumber
    r_\lambda(z)=\text{MoG}(z; \lambda)
\end{equation}
When we approximate the expectation in the KL term of Equation (\ref{equ:mle_lambda}) by the average of all given data, we can formulate the objective function as follows: 
\begin{equation}
\nonumber
    \max_\lambda \Pi_{i=1}^N r_\lambda(f(\rm{x}^{(i)}))
\end{equation}
This is the same as the objective function of the maximum likelihood estimation, so we can use EM-algorithm for the estimation of the parameter of MoG.
Further, EM-algorithm can satisfy DP by adding Gaussian noise (requirement 3), which we introduced as DP-EM \cite{park2017dp} in Section \ref{subsec:dpmethods}.

KL divergence between two mixture of Gaussian $g(\pi_a,\mu_a,\sigma_a)$ and $h(\pi_b,\mu_b,\sigma_b)$ can be approximated as follows~\cite{hershey2007approximating}:
\begin{equation}
\nonumber
\begin{split}
    &D_{\text{KL}}(g|h)\approx\\&\sum_a\pi_a\log{\frac{\Sigma_{a^\prime}\pi_{a^\prime}\exp{(-D_{\text{KL}}(\mathcal{N}(\mu_{a^\prime}, \sigma_{a^\prime})\|\mathcal{N}(\mu_{a},\sigma_{a})))}}{\Sigma_{b}\pi_{b}\exp{(-D_{\text{KL}}(\mathcal{N}(\mu_a, \sigma_a)\|\mathcal{N}(\mu_b, \sigma_b)))}}}
\end{split}
\end{equation}
Therefore, we can analytically calculate the second term of (\ref{eq:elbo_approx}) using this approximation (requirement 1).

\Add{To avoid the curse of dimensionality in the estimation of GMM, we utilize the dimensional reduction technique.}
In a dimensionality reduction, we aim to minimize the objective function (\ref{equ:error_dimensional_reduction}) with DP.
We approximate it by the average of all given data.
\begin{equation}
\nonumber
    \sum^N_{x_i}\frac{1}{N}\|x_i-g(f(x_i))\|^2_2
\end{equation}
When $f$ is a linear transformation which is useful for DP, this is optimized by PCA.
As described in Section \ref{sec:prelim}, PCA can satisfy DP (DP-PCA).
Therefore, we introduce DP-PCA as a dimensionality reduction.

\textbf{{\decodetrain}:}
As described above, since the $\mathcal{L}_\text{ELBO}$ is differentiable, we can optimize parameters with DP-SGD.
We show the pseudocode for \method in Algorithm~\ref{alg:p3gm}.
We refer to the detail of DP-SGD in Section~\ref{subsec:dpmethods}.
In Algorithm~\ref{alg:p3gm}, all parameters $\theta$ and $\phi$ are packed into $\theta$, for simplicity.

\SetKwInOut{Parameter}{Parameter}
\begin{algorithm}[t]
\caption{\method}         
\label{alg:p3gm}
\DontPrintSemicolon
\KwIn{$\mathbf{x}_1, \dots, \mathbf{x}_N \in \bf{X}$, $\boldsymbol{\mu}_X$}
\KwOut{The model parameters $\theta_T$, the parameter of MoG $\lambda$ }
\Parameter{reduced dimension $d^\prime$, privacy parameters in PCA, EM and SGD $\sigma_{p}, \sigma_e, \sigma_s$,
learning rate $\eta_t$, batch size $B$, gradient norm bound $C$}
$\Sigma$ $\leftarrow$ DP-PCA($\mathbf{x}_1, \dots, \mathbf{x}_N$;$d^\prime$, $\sigma_{p}$)\;
$\hat{\mathbf{x}}_1, \dots, \hat{\mathbf{x}}_N$ $\leftarrow$ dimensional reduction($\bf{X}, \boldsymbol{\Sigma}, \boldsymbol{\mu}_X$)\; 
$\lambda$ $\leftarrow$ DP-EM($\hat{\mathbf{x}}_1, \dots, \hat{\mathbf{x}}_N$;$\sigma_e$)\;
$p_\theta(z) \leftarrow \text{MoG}(z;\lambda)$\;
\For{$t\in [T]$}{
    Take a random batch $B_{t}$ w/ sampling probability $B/N$\;
    \textbf{Compute gradient}\;
    For each $(\mathbf{x}_b, \mathbf{\hat{x}}_b) \in [B_t]$ compute $\mathbf{g}_t(\mathbf{x}_b, \mathbf{\hat{x}}_b) \leftarrow \nabla_{\theta_t} \mathcal{L}_\text{ELBO}(\theta_t, \mathbf{x}_b, \mathbf{\hat{x}}_b, \lambda)$\;
    \textbf{Add noise and descent}\;

    $\Tilde{\mathbf{g}}_t \leftarrow \frac{1}{B}(\sum_{b \in [B_t]} \psi_C(\mathbf{g}_t(\mathbf{x}_b, \mathbf{\hat{x}}_b)) + \mathcal{N}(0,\sigma_s^2C^2\textbf{I}))$\;
    $\theta_{t+1} \leftarrow \theta_t - \eta_t \Tilde{\mathbf{g}}_t$\;
}
\Return{$\theta_T$, $\lambda$}
\end{algorithm}

\subsection{Data Synthesis using P3GM} \label{sec:synthesis}
\label{subsubsec:sample}
The data synthesis of our model follows the two steps below:
\begin{enumerate}
    \item Choose a latent vector $z$. $z \sim MoG(z;\lambda)$
    \item Generate $\Tilde{x}$ by decoding $z$. $\Tilde{x} \sim p_\theta(x|z)$.
\end{enumerate}
It is worth noting that since MoG approximates the distribution of real data, we can generate data in a similar mixing ratio of real data. 
By utilizing our model, we can share privatized data by releasing the model that satisfies DP.
Due to the post-processing properties of DP, sampled data from the model with random seeds do not violate DP.

\subsection{Privacy Analysis} \label{sec:privloss}

As we described in Section~\ref{subsec:dpmethods}, \method consumes privacy budgets at three steps: PCA, EM-algorithm, and SGD, and we introduced the differentially private methods independently. 
We can simply compute the privacy budget for each component by zCDP and MA, as described in the corresponding paper, and we can adopt sequential composition for the three components as the baseline.

Beyond this simple composition, we here follow RDP to rigorously compute the composition, and meet the following theorem.
\begin{theorem}\label{theo:comp_rdp}
\method satisfies ($\varepsilon$, $\delta$)-DP, for any $0<\delta<1$, $\alpha>1$, such that:
\begin{equation}
    \varepsilon\leq\varepsilon_p(\alpha)+T_s\varepsilon_{rs}(\alpha)+T_e\varepsilon_{re}(\alpha)+\frac{\log 1/\delta}{\alpha-1} .
    \label{theo:rdp_comp}
\end{equation}
where $\varepsilon_p(\alpha)=\alpha/(2\sigma_p^2)$, $\varepsilon_{rs}(\alpha) = MA_{\text{DP-SGD}}(\alpha-1)/(\alpha-1)$, $\varepsilon_{re}(\alpha) = MA_{\text{DP-EM}}(\alpha-1)/(\alpha-1)$, and, $T_s$ and $T_e$ are the number of iterations in DP-SGD and DP-EM, respectively. We refer to Section \ref{subsec:dpmethods} for the definition of $MA_{\text{DP-SGD}}$ and $MA_{\text{DP-EM}}$.
\end{theorem}
\begin{proof}
We consider RDP for each component of \method.
First, as described at DP-PCA in Section~\ref{subsec:dpmethods}, DP-PCA satisfies $(\alpha, \alpha/(2\sigma_p^2))$-RDP.
Second, DP-SGD satisfies $(\alpha, \varepsilon_{rs}(\alpha))$-RDP in each step from Theorem \ref{theo:relationRDP-MA} and Inequality (\ref{equ:ma-dp-sgd}).
Third, as in the case of DP-SGD, DP-EM satisfies $(\alpha, \varepsilon_{re}(\alpha))$-RDP in each step from Theorem \ref{theo:relationRDP-MA} and Inequality (\ref{equ:ma-dp-em}).
At last, from the composition theorem (Theorem \ref{theo:compositionRDP}) in RDP, P3GM satisfies $(\alpha, \varepsilon_p(\alpha)+T_s\varepsilon_{rs}(\alpha)+T_e\varepsilon_{re}(\alpha))$-RDP.
By conversion from RDP to DP from Theorem \ref{theo:relationRDPDP}, we meet (\ref{theo:rdp_comp}).
\end{proof}

\section{Discussions} \label{sec:discussion}

First, we theoretically discuss why the AEVB algorithm finds a better solution under differential privacy by adopting the two-phased training than the end-to-end training of VAE.
\Add{
Second, we discuss the parameter tuning.
}


\subsection{Solution Space Elimination}

Here, we discuss the effect to the solution space by fixing the mean of the encoder to some constant value $\mu_\phi(x)=c_x$ (In this paper, we use \Add{the dimension reduced data of $x$ by PCA} as $c_x$).
If the encoder freezes the the mean, it only searches variances to fit the posterior distribution $q_\phi(z|x)$ to the prior distribution $p_\theta(z)$.

PGM optimizes only the variance $\sigma_\phi(x)$ and the parameters of the decoder $\theta$ with the assumption that $q_\phi(z|x)$ is a Gaussian distribution whose mean is a constant $c_x$.
The loss function of PGM is as follows:
\begin{equation}
\begin{split}
\mathcal{L}_{PGM}(x)=&-D_{\text{KL}}(\mathcal{N}(c_x,\sigma_\phi(x))||p_\theta(z))\\
&+\int{\mathcal{N}(z; c_x,\sigma_\phi(x))\log p_\theta(x|z)}dz
\end{split}
\label{eq:loss_pgm}
\end{equation}

\Add{
Comparing with this, $c_x$ is not freezed, so the search space for the loss function in VAE is clearly bigger than PGM.
Assuming $\mu_\phi(x),\sigma_\phi(x)$, and $\theta$ be any values and $p_\theta(z)$ of VAE and PGM be the same, the range of all possible solutions of VAE includes all possible solutions of PGM.
}

Furthermore, assuming that $\sigma(x)$ is a constant $s_x$, we can more eliminate the search space.
The first term in Equation (\ref{eq:loss_pgm}) becomes a constant, and this is identical to autoencoder (AE) when we set $s_x=0$.
Intuitively, the decoder only learns to decode $c_x$ because the encoder is fixed to encode $x$ to $c_x$.

The elimination of the search space helps our model to discover solutions within a smaller number of iterations.
Our experiments in Section \ref{sec:exp} will demonstrate this fact.

\subsection{\Add{Parameter setting}}
\Add{
P3GM has many parameters that impact privacy and utility.
Our unique parameters that differentiate one of VAE are the reduced dimensionality and the ratio of the privacy budget allocation.
The reduced dimensionality should be higher to keep the information, but it should be small enough to estimate MoG effectively, and we found $[10, 100]$ is better.
The privacy budget allocation is also an important aspect because if the encoder fails to learn the encoding, the decoder will fail to learn the decoding, and vice verse.
Through experiments, We found that the ratio $3:7$ of the allocation to the encoder and the decoder is the better choice.
You can find more details for the experiments in Section \ref{subsec:dim} and \ref{subsec:allocation}.
However, we should have theoretically better parameters since we consume the privacy budget for the observation when we try a set of parameters to see accuracy, which remains for future work.
We note that we can use Gupta's technique to save the privacy budget~\cite{abadi2016deep, gupta2010differentially} because we do not need the output but only need accuracy to choose the better parameters.
}

\section{Experiments} \label{sec:exp}

In this section, we report the results of the experimental evaluation of \method.
For evaluating our model, we design the experiments to answer the following questions:
\begin{itemize}
    \item How effective can the generated samples be used in data mining tasks?
    \item How efficient in constructing a differentially private model?
    \item How much privacy consumption can be reduced in the privacy compositions?
\end{itemize}
To empirically validate the effectiveness of synthetic data sampled from \method, \Add{we conducted two different experiments.}

\Add{\textbf{Classification:}}
First, we train \method using a real training dataset and generate a dataset so that the label ratio is the same as the real training dataset, as we described in Section \ref{subsubsec:sample}.
Then, we train the multiple classifiers on the synthetic data and evaluate the classifiers on the real test dataset.
For the evaluation of the binary classifiers, we use the area under the receiver operating characteristics curve (AUROC) and area under the precision recall curve (AUPRC). 
For the evaluation of the multi-class classifier, we use classification accuracy.

\Add{\textbf{2-way marginals:}}
\Add{The second experiment builds all 2-way marginals of a dataset~\cite{barak2007privacy}.
We evaluate the difference between the 2-way marginals made by synthetic data and original data.
We use the average of the total variation distance of all 2-way marginals to measure the difference (i.e., the average of the half of $L_1$ distance between the two distributions on the 2-way marginals).
}

\textbf{Datasets.}
We use six real datasets as shown in Table~\ref{tab:dataset} to evaluate the performance of \method. 
Each dataset has the following characteristic:
Kaggle credit card fraud detection dataset (Kaggle Credit) is very unbalanced data which contains only $0.2\%$ positive data. 
UCI ISOLET and UCI Epileptic Seizure Recognition (ESR) dataset are higher-dimensional data whose sample sizes are small against the dimension sizes. 
Adult is a well known dataset to evaluate privacy preserving data publishing, data mining, and data synthesis.
Adult includes 15 attributes with binary class.
MNIST and Fashion-MNIST are datasets with $28 \times 28$ gray-scale images and have a label form 10 classes. 
We use $90\%$ of the datasets as training datasets and the rest as test datasets.

\begin{table}
\centering
\small
\begin{tabular}{lrrrr}
    \toprule
    Dataset                                 & $N$       & \#feature & \#class   & \%positive \\ \midrule
    Kaggle Credit~\cite{dal2015calibrating} & 284807    & 29        & 2         & 0.2 \\
    Adult \protect \footnotemark            & 45222     & 15        & 2         & 24.1 \\
    UCI ISOLET \protect\footnotemark        & 7797      & 617       & 2         & 19.2 \\
    UCI ESR \protect\footnotemark           & 11500     & 179       & 2         & 20.0 \\
    MNIST                                   & 70000     & 784       & 10        & -  \\
    Fashion-MNIST                           & 70000     & 784       & 10        & -  \\ \bottomrule
\end{tabular}
\caption{Datasets}
\label{tab:dataset}
\end{table}
\footnotetext[4]{https://archive.ics.uci.edu/ml/datasets/adult}
\footnotetext[5]{https://archive.ics.uci.edu/ml/datasets/Epileptic+Seizure+Recognition}
\footnotetext[6]{https://archive.ics.uci.edu/ml/datasets/isolet}

\begin{table}[t!]
\small
\centering
\begin{tabular}{llcrr}
\toprule
              & $\sigma_s$ & learning rate & \#epochs   & batch size \\ \midrule
Kaggle Credit & 2.1       & 0.001         & 15         & 100  \\
Adult         & 1.4        & 0.001          & 5           & 200  \\
UCI ISOLET    & 1.6       & 0.001          & 2          & 100  \\
UCI ESR       & 1.4        & 0.001         & 2          & 100  \\
MNIST         & 1.4       & 0.001         & 4          & 300  \\
Fashion MNIST & 1.4       & 0.001         & 4          & 300  \\ \bottomrule
\end{tabular}
\caption{Hyper-parameters for each dataset.}
\label{tab:hyperparams}
\end{table}

\textbf{Implementations of Generative Models.}
The encoder has two FC layers of [$d, 1000, d^\prime$] with ReLU as the activate function. $d$ is the dimensionality of data and $d^\prime$ is the reduced dimensionality. 
The decoder also has two FC layers with [$d^\prime, 1000, d$] with ReLU as the activate function. 
We show the hyper-parameters used in DP-SGD for each dataset in Table~\ref{tab:hyperparams}.
For the Kaggle Credit dataset, we did not apply dimensionality reduction because this dataset is originally low dimensionality. 
For the other datasets, we did dimensionality reduction with reduced dimensionality $d_p=10$ and $\varepsilon_p=0.1$. 
We set $\sigma_e$ as $\varepsilon=1$ holds, $T_e=20$, \Add{the number of components of MoG as $d_m=3$, and we use the diagonal covariance matrix as the covariance matrix of MoG for the efficiency.} 
We develop the above models by Python 3.6.9 and PyTorch 1.4.0 \cite{paszke2017automatic}.

\textbf{Implementations of Classifiers.}
For table datasets, we use four different classifiers, LogisticRegression (LR), AdaBoostClassifier (AB)~\cite{freund1996experiments}, GradientBoostingClassifier (GBM)~\cite{friedman2001greedy}, and XgBoost (XB)~\cite{chen2016xgboost} from Python libraries, scikit-learn 0.22.1, and xgboost 0.90. 
We set the parameters of sklearn.GradientBoostingClassifier as max\_features="sqrt", max\_depth=8, min\_samples\_leaf=50 and min\_samples\_split=200.
Other parameters are set to default.
For image datasets, we train a CNN for the classification tasks using Softmax. 
The model has one Convolutional network with 28 kernels whose size is (3,3) and MaxPooling whose size is (2,2) and two FC layers with [128, 10]. 
We use ReLU as the activate function and apply dropout in FC layers.

\textbf{Competitors.}
We compare \method with PrivBayes \cite{zhang2014privbayes}, DP-GM \cite{acs2018differentially}, \Add{Ryan's algorithm~\cite{kenna2019nistwinner}} and DP-VAE.

\textbf{Reproducibility.}
We will make the code public for reproducibility.
Under the review process, our code is available on the anonymous repository\footnote{\url{https://github.com/tkgsn/P3GM}}.

\begin{table}
\centering
\small
\begin{tabular}{@{}lllllllll@{}}
    \toprule
    \multirow{2}{*}{}   & \multicolumn{3}{c}{AUROC} & \multicolumn{3}{c}{AUPRC} \\  
                        & VAE    & PGM   & \textbf{P3GM}   & VAE    & PGM  & \textbf{P3GM}   \\ 
                        \cmidrule(lr){1-1} \cmidrule(lr){2-4} \cmidrule(lr){5-7}
    LR                     & 0.9617 & 0.9454   & 0.9264 & 0.6542 & 0.6865  & 0.6750 \\
    AB    & 0.9599 & 0.9330   & 0.9026 & 0.5737 & 0.6528  & 0.6474 \\
    GBM           & 0.9619 & 0.9442   & 0.9182 & 0.6838 & 0.6734  & 0.6645 \\
    XB          & 0.9395 & 0.9321   & 0.9026 & 0.2745 & 0.6469  & 0.6218 \\ \bottomrule
\end{tabular}
\caption{Accuracy comparison with non-private models. PGM and \method show relatively close accuracy against VAE.}
\label{tab:exp_credit}
\end{table}


\begin{table*}[t!]
\small
\centering
\begin{tabular}{@{}lllllllllll@{}}
\toprule
\multicolumn{1}{l}{\multirow{2}{*}{Dataset}} & \multicolumn{5}{c}{AUROC} & \multicolumn{5}{c}{AUPRC}\\
\multicolumn{1}{c}{} & PrivBayes & {\Add {Ryan's}} & DP-GM   & \textbf{P3GM} & original & PrivBayes & {\Add {Ryan's}} & DP-GM    & \textbf{P3GM} & original \\ \cmidrule(l){1-1} \cmidrule(l){2-6}\cmidrule(l){7-11}
Kaggle Credit & 0.5520 & 0.5326 & 0.8805     & \textbf{0.8991}     & 0.9663   & 0.2084 & 0.2503 & 0.3301     & \textbf{0.6586}        & 0.8927   \\
UCI ESR                                       & 0.5377 & 0.5757 & 0.4911     & \textbf{0.8801}      & 0.8698 & 0.5419 & 0.4265 & 0.3311     & \textbf{0.7672}        & 0.8098   \\
Adult                                         & \textbf{0.8530} & 0.5048 & 0.7806     & 0.8214        & 0.9119   & \textbf{0.6374} & 0.2584 & 0.4502    & 0.5972        & 0.7844   \\
UCI ISOLET                                    & 0.5100 & 0.5326 & 0.4695     & \textbf{0.7498}        & 0.9891   & 0.2084 & 0.2099 & 0.1816    & \textbf{0.3950}        & 0.9623   \\ \bottomrule
\end{tabular}
\caption{Performance comparison on four real datasets. Each score is the average AUROC or AUPRC over four classifiers listed in Table \ref{tab:exp_credit}. \method outperforms other two differentially private models on three datasets.}
\label{tab:compara_exist}
\end{table*}


\subsection{Effectiveness in Data Mining Tasks}

We evaluate how effective can the generated samples be used in several data mining tasks.
We also empirically evaluate the trade-off between utility and the privacy protection level.

\textbf{Against non-private models in table data.}
Here, we show that P3GM with ($1,10^{-5}$)-DP does not cause much utility loss than non-private models: PGM and VAE.
Table~\ref{tab:exp_credit} presents the results on the Kaggle Credit dataset. 
As listed in Table~\ref{tab:exp_credit}, we utilized four different classifiers. 
Comparing PGM with VAE, it is said that PGM has similar expression power as VAE. 
Comparing P3GM with non-private methods, in spite of the noise for DP, we can see that scores of P3GM do not significantly decrease, which shows the tolerance to the noise.

\textbf{Comparison with private models in table data.}
Next, we perform comparative analysis for \method, PrivBayes, and DP-GM with ($1,10^{-5}$)-DP on four real datasets. 
In Table~\ref{tab:compara_exist}, we give the performance on each dataset averaged across these four different classifiers as well as Table~\ref{tab:exp_credit}. 
\method outperforms the other two differentially private models on three datasets in AUROC and AUPRC.
However, there is much degradation on the UCI ISOLET dataset. 
This is because the smaller data size causes more noise for DP and the high dimensionality makes it difficult to find a good solution in the small number of iterations.
On the Adult dataset, PrivBayes shows a little better performance than P3GM due to the simpler model.
However, PrivBayes only performs well for datasets having simple dependencies and a small number of features like the Adult dataset.
Regarding the high dimensional data, our method is significantly better than PrivBayes.
\Add{Ryan's algorithm does not perform well on this task.
Since Ryan's algorithm constructs data from limited information (i.e., 1, 2, and 3-way marginals), Ryan's algorithm cannot reconstruct the essential information which is required for the classification task.
}

\textbf{Classification on image datasets.}
We perform comparative analysis with DP-GM and PrivBayes on MNIST and Fashion-MNIST, whose data is high-dimensional.
In Table~\ref{tab:exp_mnist}, we give the classification accuracies of classifiers trained on each synthetic data.
P3GM results in much better results than DP-GM and PrivBayes, which shows the robustness to high-dimensional data.
Moreover, \method shows around 6\% and 5\% less accuracy than VAE on MNIST and Fashion-MNIST, respectively, under ($\varepsilon, \delta$)=($1,10^{-5}$).
The accuracies are relatively close to VAEs even P3GM satisfied the differential privacy.
Back to Figure \ref{fig:crown_jewel}, we also displayed generated samples from our model and competitors.
As we can see, \method can generate images that are visually closer to VAE while satisfying differential privacy.

\begin{table}[t]
\small
\centering
\begin{tabular}{@{}llllll@{}}
\toprule
Dataset       & VAE  & DP-GM & PrivBayes  & {\Add{Ryan's}} &  \textbf{P3GM}   \\ \midrule
MNIST         & 0.8571                 & 0.4973 & 0.0970 & 0.2385 & \bf{0.7940} \\
Fashion & 0.7854                 & 0.5200 & 0.0996 & 0.2408 & \bf{0.7485} \\ \bottomrule
\end{tabular}
\caption{Classification accuracies on image datasets.}
\label{tab:exp_mnist}
\end{table}

\begin{figure}[t]
    \centering    
    \subfloat[AUPRC]{
        \includegraphics[width=0.495\hsize]{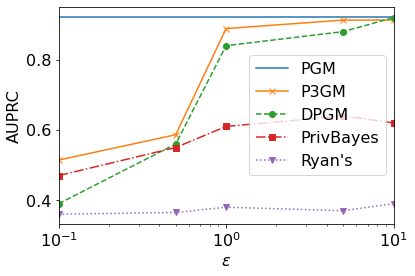}
        \label{fig:auroc_kaggle}
    }
    \subfloat[AUROC]{
        \includegraphics[width=0.495\hsize]{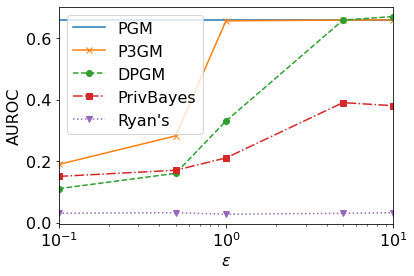}
        \label{fig:auprc_kaggle}
    }
    \caption{\Add{Utility in fraud detection (Kaggle Credit).}}
    \label{fig:exp_credit_wrt_eps}
\end{figure}

\textbf{Varying privacy levels.}
Here, we measure our proposed method's performance when we vary the privacy protection level $\varepsilon$.
In Figure~\ref{fig:exp_credit_wrt_eps}, we plot AUROC and AUPRC of classifiers trained using synthetic Kaggle Credit dataset generated by \method, DP-GM and PrivBayes w.r.t. each $\varepsilon$ with $\delta=10^{-5}$.
PrivBayes does not show high scores even when $\varepsilon$ is large, which means that PrivBayes does not have enough capacity to generate datasets whose dependencies of attributes are complicated, such as Kaggle Credit.
Also, as we can see, although DP-GM rapidly degrades the scores as $\varepsilon$ becomes smaller, the ones of P3GM does not significantly decrease.
This result shows that P3GM is not significantly influenced by the noise for satisfying DP.

\textbf{Varying degree of dimensionality reduction.}
Here, we empirically explain how the number of components of PCA affects the performance of \method. 
Figure~\ref{fig:pca_perform} shows the results on the MNIST.
As we can see, the number of components ($d_p$) affects performance. 
Too much high dimensionality makes (DP-)EM algorithm ineffective due to the curse of dimensionality. 
Too much small dimensionality lacks the expressive power for embedding.
From the result, $d_p=[10,100]$ looks a good solution with balancing the accuracy and the dimensionality reduction on the MNIST dataset.

\begin{figure}[t]
\begin{minipage}{0.49\hsize}
   \centering
   \includegraphics[width=\hsize]{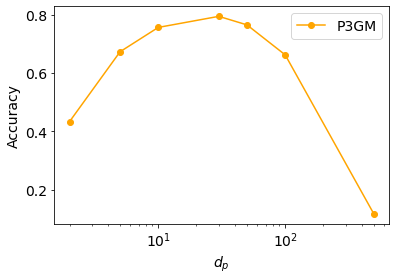}
   \caption{Reducing dimension improves accuracy (MNIST).}
   \label{fig:pca_perform}
\end{minipage}
\hfill
\begin{minipage}{0.49\hsize}
   \centering
   \includegraphics[width=\hsize]{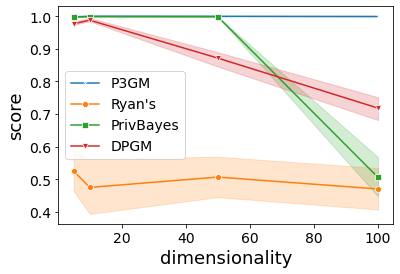}
   \caption{\Add{Only P3GM can handle high-dimensionality.}}
   \label{fig:dimensonality}
\end{minipage}
\end{figure}

\begin{figure*}[t]
    \centering    
    \subfloat[Reconstruction loss (MNIST).]{
        \includegraphics[width=0.24\hsize]{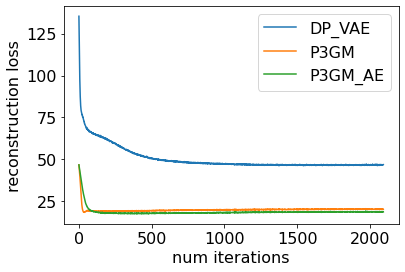}
        \label{fig:loss_mnist}
    }
    \subfloat[Reconstruction loss (Kaggle Credit)]{
        \includegraphics[width=0.24\hsize]{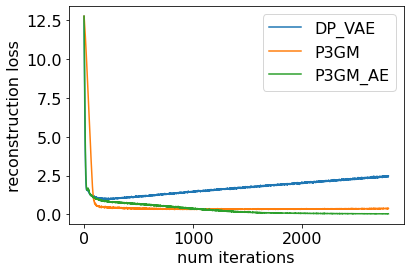}
        \label{fig:loss_credit}
    }
    \subfloat[Classification accuracy (MNIST).]{
        \includegraphics[width=0.24\hsize]{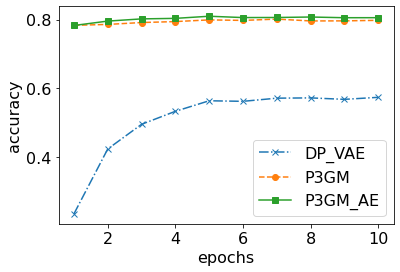}
        \label{fig:epoch_mnist}
    }
    \subfloat[AUROC (Kaggle Credit)]{
        \includegraphics[width=0.24\hsize]{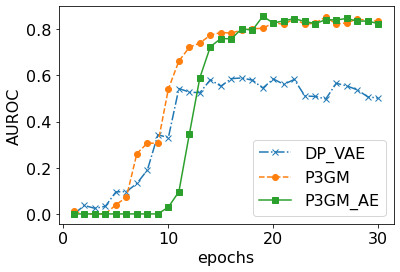}
        \label{fig:epoch_credit}
    }
    \caption{\method demonstrates higher learning efficiency than DP-VAE. More simple model increases more learning efficiency.}
    \label{fig:efficiency}
\end{figure*}

\subsection{Learning Efficiency}

Here, we measure the learning efficiency of the proposed method.
As discussed in Section \ref{sec:discussion}, we can interpret that our model reduces the search space to accelerate the convergence speed.
We empirically demonstrate it in Figure \ref{fig:efficiency}.
Here, let \method(AE) denote \method with fixing $\sigma_\phi(x)=0$. 

In Figure \ref{fig:loss_mnist} and Figure \ref{fig:loss_credit}, our proposed method shows faster converegence than na\"ive method (DP-VAE) in the reconstruction loss  (the first term of (\ref{eq:elbo_approx})).
For these two datasets, \method met convergences at earlier epochs than DP-VAE.
In Figure \ref{fig:loss_credit}, the loss of DP-VAE is decreased gradually in the long term but shows fluctuations in the short term.
In contrast, \method shows monotonic decreases in reconstruction loss.
This is due to the solution space elimination by freezing the encoder in our model.

We plot the performance in each epoch to see the convergence speed of the models in Figure \ref{fig:epoch_mnist} and Figure \ref{fig:epoch_credit}, which shows that the performance with the smaller search space also converges faster.
Figure \ref{fig:epoch_mnist} shows the classification accuracy with MNIST dataset and Figure \ref{fig:epoch_credit} shows the AUROC with the Kaggle Credit dataset.
In both results, \method(AE) converged at the earliest iteration in those three methods.
While at the end of iterations, \method shows the best results, and \method(AE) is the second-best.
This is because the search space of P3GM is larger than P3GM (AE), so it can find the better solution.
In a similar vein, VAE can find a better solution since the search space of VAE is larger than P3GM, but it will cost a non-acceptable privacy budget.
P3GM balances the search space size and the cost of the privacy budget.

\begin{table}[t]
\small
\centering
\begin{tabular}{@{}lllll@{}}
\toprule
\Add{Dataset}       & \Add{PrivBayes} & \Add{Ryan's} & \Add{DP-GM} & \Add{P3GM}   \\ \midrule
\Add{Kaggle Credit} & \Add{0.9326}    & {\Add{\textbf{0.0082}}} & \Add{0.1345}      & \Add{0.2411} \\
\Add{UCI ESR}       & \Add{0.8793}    & {\Add{0.0936}} & \Add{0.6083}      & \Add{\textbf{0.0654}} \\
\Add{Adult}         & \Add{0.0752}    & {\Add{\textbf{0.0494}}} & \Add{0.2672}      & \Add{0.3867} \\
\Add{UCI ISOLET}    & \Add{0.5324}    & {\Add{\textbf{0.1611}}} & \Add{0.6855}      & \Add{0.3029} \\ \bottomrule
\end{tabular}
\caption{\Add{Avg. total variational distance of 2-way marginals.\label{tab:2way}}}
\end{table}

\subsection{\Add{Accuracy of 2-way marginals}}
\Add{Table~\ref{tab:2way} shows the results of the 2-way marginals experiment, which is one of the NIST competition measures.
Here, since the total variation distance represents the error of the distribution (of 2-way marginals) made by the synthetic data, we can see that Ryan's algorithm performs the best.
This is because Ryan's algorithm directly computes the crucial 2-way marginals with the Gaussian mechanism, which leads to the win in the NIST competition.
The direct method, such as Ryan's algorithm, is useful for a simple task like building a 2-way distribution. 
Still, it loses the complex information such as the dependencies in multiple attributes, which results in the worse result in the classification task.
Although P3GM is inferior to Ryan's algorithm for this task, P3GM keeps the 2-way marginal distribution better than PrivBayes and DP-GM for datasets with high dimensionalities.
}

\subsection{{\Add{Dimensionality}}}
\label{subsec:dim}
\Add{We here experimentally show the impact to the algorithms by a dimensionality.
We make a dataset by sampling from two Gaussian distributions whose means are $(1,1 , \dots,1)$ and $(-1, -1, \dots, -1)$ and whose covariance matrices are the identity matrices and attach different labels.
From the two Gaussians, we generate synthesized datasets with $\epsilon=1$ varying the dimensionality from $5$ to $100$, and measure the AUROC in binary classifications for the datasets.
Figure~\ref{fig:dimensonality} shows these scores of each method.
P3GM can handle high-dimensinality, but the others are highly depending on the dimensionality.}

\subsection{{\Add{Privacy budget allocation}}}
\label{subsec:allocation}
\Add{P3GM has two components: Encoding Phase (i.e., DP-PCA and DP-EM) and Decoding Phase (i.e., DP-SGD).
So far, we used the fixed privacy budget allocation for the two components.
Here, we experimentally explore the P3GM's perfomance variations when varying the ratio of the allocation while keeping the total privacy budget to $1$.
We plot the result for the Adult dataset in Figure~\ref{fig:budget_allocation}.
We can see that the score is top when the ratio is from $0.1$ to $0.3$\footnote{This is lower than the case of the Wishart PCA algorithm (about the Wishart PCA, refer to the footnote of the first page). One can see that the more noisy PCA algorithm requires more budget for the reconstruction to compensate the PCA algorithm.}.
However, it is interesting to explore a theoretically better allocation ratio rather than this empirical result, which remains for future work.
}

\begin{figure}[t]
\begin{minipage}{0.49\hsize}
   \centering
   \includegraphics[width=\hsize]{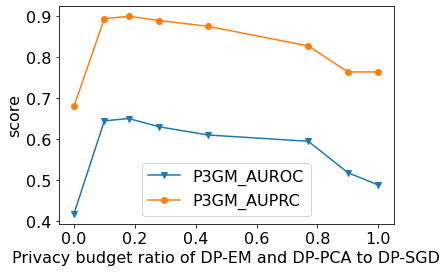}
   \caption{\Add{Performance in varying privacy budget allocations.}}
   \label{fig:budget_allocation}
\end{minipage}
\hfill
\begin{minipage}{0.49\hsize}
  \centering\includegraphics[width=\hsize]{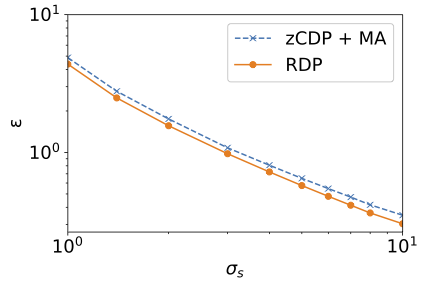}
  \caption{Privacy composition in RDP saves privacy budgets.}
  \label{fig:privacy_accountant}
\end{minipage}
\end{figure}

\subsection{Privacy Composition}

In this paper, we introduced the composition of privacy loss based on RDP.
Here, we show that our composition method more rigorously accountants each privacy budget than the baseline.
We use zCDP for DP-EM and MA for DP-SGD as the baseline, which is proposed composition methods in the corresponding papers.
Figure \ref{fig:privacy_accountant} shows the computed value of $\varepsilon$ by each method, varying the amount of noise for DP-SGD.
We freeze the amount of noise for MA.
Our composition based on RDP results in a smaller value of $\epsilon$ than the baseline.
Thus our method can compute the privacy composition in a lean way.

\section{Conclusion}


This paper addressed the question, how can we release a massive volume of sensitive data while mitigating privacy risks?
Particularly, to construct a differentially private deep generative model for high dimensional data, we introduced a novel model, \method.
The proposed model \method hires an encoder-decoder framework as well as VAE but employs a different algorithm that introduces a two-phase process for training the model to increase the robustness to the differential privacy constraint.
We also gave a theoretical analysis of how effectively our method reduces complexity comparing with VAE.
We further provided an extensive experimental evaluation of the accuracy of the synthetic datasets generated from \method.
Our experiments showed that data mining tasks using data generated by \method are more accurate than existing techniques in many cases.
The experiments also demonstrated that \method generated samples with less noise and \Add{resulted in higher utility in classification tasks} than competitors.
\Add{
Exploring the the optimality of parameters, dimensionality reduction (Equation \ref{equ:error_dimensional_reduction}), and estimation of the latent distribution (Equation \ref{equ:mle_lambda}) remains for future works.
}


\bibliographystyle{abbrv}
\bibliography{ref}



\end{document}